\documentclass[10pt,twocolumn,letterpaper]{article}

\usepackage{cvpr}              

\usepackage{graphicx}
\usepackage{amsmath}
\usepackage{amssymb}
\usepackage{amsthm}
\usepackage{booktabs}
\usepackage[ruled,vlined]{algorithm2e}
\usepackage{color}
\usepackage{footmisc}
\newcommand{\PyComment}[1]{\ttfamily\textcolor{commentcolor}{\# #1}}  
\newcommand{\PyCode}[1]{\ttfamily\textcolor{black}{#1}} 
\definecolor{commentcolor}{RGB}{110,154,155}
\newtheorem{theorem}{Theorem}
\newtheorem{proposition}[theorem]{Proposition}
\usepackage[pagebackref,breaklinks,colorlinks]{hyperref}
\usepackage[accsupp]{axessibility}

\usepackage[capitalize]{cleveref}
\crefname{section}{Sec.}{Secs.}
\Crefname{section}{Section}{Sections}
\Crefname{table}{Table}{Tables}
\crefname{table}{Tab.}{Tabs.}

\def\cvprPaperID{9799} 
\def\confName{CVPR}
\def\confYear{2023}

\begin{document}
	
	\title{Transforming  Radiance Field with Lipschitz Network for 
		
		Photorealistic 3D Scene Stylization}
	
 \author{%
 	Zicheng Zhang\textsuperscript{1}
	\quad
	Yinglu Liu\textsuperscript{2}
	\quad
	Congying Han\textsuperscript{1}\thanks{Corresponding author}
	\quad
        Yingwei Pan\textsuperscript{2}
        \quad
	Tiande Guo\textsuperscript{1}
	\quad
	Ting Yao\textsuperscript{2}
	\quad
	\\ 
	\textsuperscript{1}University of Chinese Academy of Sciences
	\quad
	\textsuperscript{2}JD AI Research
	\quad
	\vspace{.5em} 
	\\
	\tt\small zhangzicheng19@mails.ucas.ac.cn
	\quad
	liuyinglu1@jd.com
	\quad
	hancy@ucas.ac.cn
	\\
	\tt\small
     panyw.ustc@gmail.com
	\quad
	tdguo@ucas.ac.cn
	\quad
	tingyao.ustc@gmail.com
}
	\maketitle
	
	\begin{abstract}
		Recent advances in 3D scene representation and novel view synthesis have witnessed the rise of Neural Radiance Fields (NeRFs). Nevertheless, it is not trivial to exploit NeRF for the photorealistic 3D scene stylization task, which aims to generate visually consistent and photorealistic stylized scenes from novel views.
		Simply coupling NeRF with photorealistic style transfer (PST) will result in cross-view inconsistency and degradation of stylized view syntheses.
		Through a thorough analysis, we demonstrate that this non-trivial task can be simplified in a new light: \textbf{When transforming the appearance representation of a pre-trained NeRF with Lipschitz mapping, the consistency and photorealism across source views will be seamlessly encoded into the syntheses.}
		That motivates us to build a concise and flexible learning framework namely LipRF, which upgrades arbitrary 2D PST methods with Lipschitz mapping tailored for the 3D scene.
		Technically, LipRF first pre-trains a radiance field to reconstruct the 3D scene, and then emulates the style on each view by 2D PST as the prior to learn a Lipschitz network to stylize the pre-trained appearance.
		In view of that Lipschitz condition highly impacts the expressivity of the neural network, we devise an adaptive regularization to balance the reconstruction and stylization.
		A gradual gradient aggregation strategy is further introduced to optimize LipRF in a cost-efficient manner.
		We conduct extensive experiments to show the high quality and robust performance of LipRF on both photorealistic 3D stylization and object appearance editing. 
	\end{abstract}
	
	\begin{figure*}[t]
		\centering
		\vspace{-1cm}
		\includegraphics[width=0.98\linewidth]{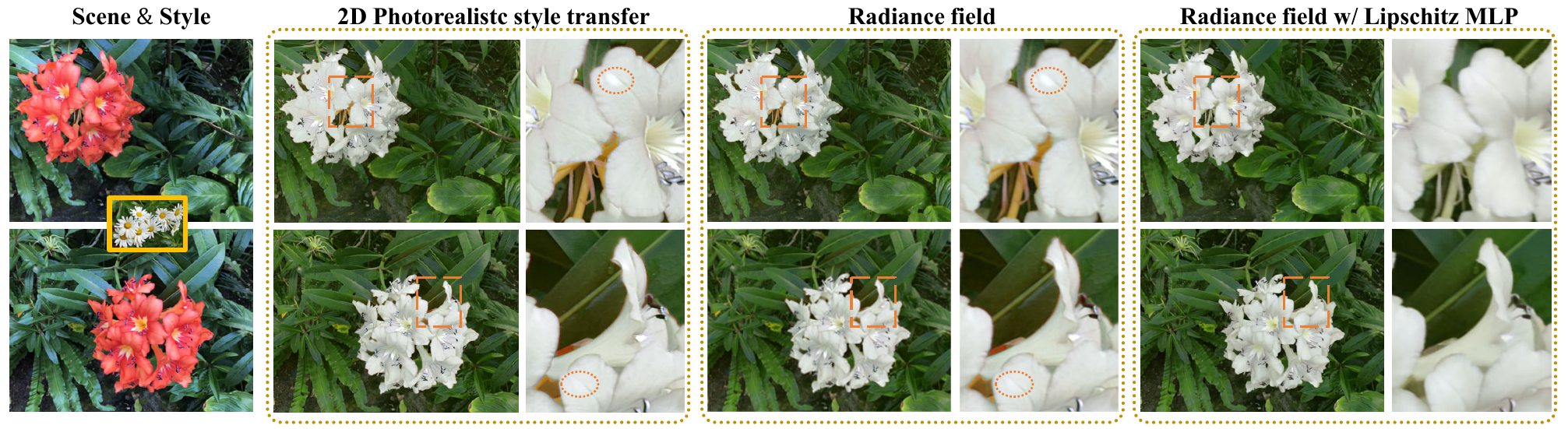}
		\vspace{-0.3cm}
		\caption{Illustrations of different strategies for photorealistic 3D scene stylization.  The 2D PST method \cite{Chiu2022PhotoWCT2CA}  generates disharmonious orange color on the stem and petaline edge. The regions bounded by the dotted ellipses are inconsistent due to the white spot in the first view. When employing a radiance field to reconstruct the results of 2D PST method, the disharmony and inconsistency are still retained. Our LipRF successfully eliminates these downsides, and renders high-quality stylized view syntheses.
		}
		\vspace{-0.5cm}
		\label{fig: motivation}
	\end{figure*}

    
	\section{Introduction}\label{sec:intro}
	Photorealistic style transfer (PST) \cite{ErikReinhard2001ColorTB} is one of the important tasks for visual content creation, which aims to automatically apply the color style of a reference image to another input (\textit{e.g.}, image \cite{FujunLuan2017DeepPS} or video \cite{JaejunYoo2019PhotorealisticST}). In this task, the stylized result is required to look like a camera shot and preserve the input structure (\textit{e.g.}, edges and regions). Benefiting from the launch of deep learning, a series of sophisticated deep PST methods \cite{FujunLuan2017DeepPS,YijunLi2018ACS,JaejunYoo2019PhotorealisticST,huo2021manifold,Wu2022CCPLCC} have been developed for practical usage. Recent progress in 3D scene representation has featured Neural Radiance Field \cite{Mildenhall2020NeRFRS,AlexYu2022PlenoxelsRF} (NeRF) with efficient training and high-quality view synthesis.
	This inspires us to go one step further to explore a more challenging task of photorealistic 3D scene stylization, which is to generate visually consistent and photorealistic stylized syntheses from arbitrary views.
	Such a task enables an automatic modification of 3D scene appearance with different lighting, time of day, weather, or other effects, thereby enhancing user experience and stimulating emotions for virtual reality \cite{Tseng2022Pseudo3DSM}.
	
	Nevertheless, it is not trivial to build an effective framework for photorealistic 3D scene stylization. The difficulty mainly originates from the fact that there is no valid photorealistic style loss tailored for training NeRF. In general, the image-based PST is commonly tackled via either the neural style transfer \cite{Gatys2016ImageST} combined with complicated post-processing \cite{FujunLuan2017DeepPS,Mechrez2017PhotorealisticST,YijunLi2018ACS}, or particular network structures \cite{YijunLi2017UniversalST,JaejunYoo2019PhotorealisticST,Xia2020JointBL}.  However, none of them can be directly applied to the learning of NeRF. As shown in Figure \ref{fig: motivation}, simply employing the state-of-the-art 2D PST on each view might result in noise, disharmony and even inconsistency across views, since the PST methods rely on the size or object masks of the inputs. Such downsides will be further amplified after reconstructing the 3D scene with NeRF.
	
	To alleviate these limitations, we start with a basic understanding of this task: \textit{Though preserving the photorealism and consistency seems to be different in the context of 2D images, they do have the same essence when moving to 3D volume rendering} \cite{Max1995OpticalMF}. From this standpoint, the task is simplified as a problem to regulate the volume rendering variance of the radiance field before and after stylization. According to the studies of color mapping \cite{ErikReinhard2001ColorTB,Omer2004ColorLI,Piti2007TheLM}, some specific linear mappings of image pixels can nicely preserve the image structures with photorealistic effect. Motivated by this, we theoretically demonstrate that a simple yet effective design of Lipschitz-constrained linear mapping over appearance representation can elegantly control the volume rendering variance. Furthermore, we prove that replacing linear mapping with a Lipschitz multilayer perceptron (MLP) also holds these properties under extra assumptions which can be relaxed in practice. Such a way completely eliminates the drawbacks of 2D PST when transforming the radiance field with a Lipschitz MLP (see Figure \ref{fig: motivation}). In a nutshell, our analysis verifies that Lipschitz MLP can be interpreted as an implicit regularization to safeguard the 3D photography of stylized scenes.
	
	By consolidating the idea of transforming the radiance field with the Lipschitz network, we propose a novel NeRF-based architecture (namely LipRF) for photorealistic 3D scene stylization. Technically, LipRF contains two stages: 1) training a radiance field to reconstruct the source 3D scene; 2) learning a Lipschitz network to transform the pre-trained appearance representation to the stylized 3D scene with the guidance of style emulation on each view by arbitrary 2D PST. We adopt the Plenoxels \cite{AlexYu2022PlenoxelsRF}  as the base radiance field due to its advanced reconstruction quality and compressed appearance representation by spheric harmonics. Considering that the Lipschitz
	condition greatly impacts the expressivity of neural networks, we design an adaptive regularization based on spectral normalization
	\cite{Miyato2018SpectralNF} to allow a mild relaxation of the Lipschitz constant in each linear layer.  Finally,  we capitalize on gradual gradient aggregation to optimize LipRF in a cost-efficient fashion.
	
	In summary, we have made the following contributions: (\textbf{I}) We present a thorough and insightful analysis of photorealistic 3D scene stylization on the basis of volume rendering and Lipschitz transformation. (\textbf{II}) We build a concise and flexible framework (LipRF) for photorealistic 3D scene stylization by novelly transforming the appearance representation of a pre-trained NeRF with the Lipschitz Network. (\textbf{III}) Under the Lipschitz condition, we design adaptive regularization and gradual gradient aggregation to seek a better trade-off among the reconstruction, stylization quality, and computational cost. We evaluate LipRF on both photorealistic 3D stylization and object appearance editing tasks to validate the effectiveness of our proposal. 
	
	\section{Related works}
	
	\subsection{Novel view synthesis}
    Novel view synthesis aims to synthesize view images with arbitrary camera poses from given source images. Many works have been proposed on various discrete representations, \textit{e.g.}, multi-plane image \cite{BenMildenhall2019LocalLF,TinghuiZhou2018StereoML}, point clouds \cite{ChenHsuanLin2017LearningEP,GkioxariGeorgia2019SynSinEV}, meshes\cite{PaulDebevec1996ModelingAR,MichaelWaechter2014LetTB}, and voxels \cite{StephenLombardi2019NeuralVL}. Recently, neural radiance field approaches \cite{Mildenhall2020NeRFRS,Zhang2020NeRFAA,MartinBrualla2021NeRFIT,Barron2021MipNeRFAM} encoder the scene into a continuous implicit volumetric representation via multilayer perceptron, and render the novel view by volume rendering integral \cite{Max1995OpticalMF}. Later, Plenoxels \cite{AlexYu2022PlenoxelsRF} reveals that the key element of NeRF is the differentiable volume renderer. By simplifying NeRF into a sparse voxel grid with spherical harmonics, Plenoxels achieves substantial speedups with comparable rendering quality. Our work capitalizes on the benefits of Plenoxels, and enables photorealistic 3D scene stylization to achieve in a few minutes.

	\subsection{Style transfer methods}
	
	\noindent\textbf{Image stylization.}
    Photorealistic style transfer is a long-standing topic \cite{ErikReinhard2001ColorTB,Omer2004ColorLI,FranoisPiti2005NdimensionalPD,SoonminBae2006TwoscaleTM,Chen2016BilateralGU} 
	that focuses on color distribution transfer while maintaining the photorealism of the image.  In the regime of deep learning, since Gatys \textit{et al.} \cite{Gatys2016ImageST} present the neural style transfer that matches the feature distribution \cite{YanghaoLi2017DemystifyingNS,LiPan2019OptimalTO,chen2019mocycle,Kolkin2019StyleTB,Kalischek2021InTL,Zhang_2022_CVPR} to transfer artistic texture, many works incorporate neural networks into PST for more complicated visual effects. One of the early works, Luan \textit{et al.} \cite{FujunLuan2017DeepPS} introduces the semantic mask and photorealistic regularization \cite{AnatLevin2006ACF} to yield impressive results. PhotoWCT \cite{YijunLi2018ACS} further improves the regularization \cite{AnatLevin2006ACF} into a closed form. The works \cite{Mechrez2017PhotorealisticST,JaejunYoo2019PhotorealisticST,Li2019HighResolutionNF,Chiu2022PhotoWCT2CA} propose to preserve the high frequency to ensure photorealism.  
	For example
	, WCT$^2$ \cite{JaejunYoo2019PhotorealisticST} embeds the wavelet transform into neural networks to fully retain image structure.  Other works adopt particular style transfer operators \cite{YijunLi2017UniversalST,Lu2019ACS,huang2017arbitrary,Li2019LearningLT,huo2021manifold,Wu2022CCPLCC} to the Encoder-Decoder pre-trained on natural image dataset, \textit{e.g.}, COCO \cite{TsungYiLin2014MicrosoftCC}.  The advances \cite{Xia2020JointBL,Xia2021RealtimeLP} learn the edge-preserved local affine grid  \cite{Gharbi2017DeepBL} in bilateral space \cite{Chen2016BilateralGU} to transfer the color locally. Because the input size and semantic mask heavily influence these methods, directly applying to 3D scene will produce distortion and inconsistency.
	
	\noindent\textbf{3D scene stylization.}
	Recently, several works \cite{PeiZeChiang2021Stylizing3S,ThuNguyenPhuoc2022SNeRFSN,Huang2022StylizedNeRFC3,Zhang2022ARFAR} couple NeRF with neural style transfer \cite{Gatys2016ImageST} to tackle artistic 3D scene stylization. For example, \cite{PeiZeChiang2021Stylizing3S,ThuNguyenPhuoc2022SNeRFSN, Zhang2022ARFAR} finetune pre-trained NeRFs with the image-based style losses. \cite{Huang2022StylizedNeRFC3} adapts the pre-trained NeRF by means of a 2D style transfer model.   These methods enable better stylization in a variety of scenes, and perform faster training and inference than the approaches \cite{HsinPingHuang2021LearningTS,FangzhouMu20213DPS,KangxueYin20223DStyleNetC3}. However, strict consistency and photorealism are hardly achieved by these methods, making them inapplicable to photorealistic stylization.
	
	\subsection{Lipschitz network}
	The Lipschitz network \cite{Scaman2018LipschitzRO}, \textit{i.e.}, neural network with a limited Lipschitz constant, has advantages in robustness \cite{YaoYuanYang2020ACL}, generalization \cite{YuichiYoshida2017SpectralNR}, and stability of training \cite{Miyato2018SpectralNF}. Since computing the exact Lipschitz constant of neural networks is NP-hard \cite{Scaman2018LipschitzRO}, previous methods usually minimize \cite{Gouk2021RegularisationON,Liu2022LearningSN} or specify \cite{Miyato2018SpectralNF,Gulrajani2017ImprovedTO} an upper bound of Lipschitz constant as a small value. Due to the small Lipschitz constant limits the visual effect in synthesis task, these methods cannot provide a proper  constraint for the renderer/generator. To solve this, we propose to adaptively constrain Lipschitz property according to the given scene and reference image.
	
	\section{Preliminaries}
	
	\noindent\textbf{Radiance field.}
	Generally, the radiance field \cite{Mildenhall2020NeRFRS} is a 5D function $\mathbf{F}$ that maps any 3D location $\boldsymbol{x}$ and viewing direction $\boldsymbol{d}$ to volume density $\sigma$ and color $\boldsymbol{c} = (r,g,b)$. Specifically, $\mathbf{F}$ can be further divided into the geometry part $\sigma = \mathbf{F}_{geo}(\boldsymbol{x})$ and the appearance part $\boldsymbol{c} = \mathbf{F}_{app}(\boldsymbol{x},\boldsymbol{d})$, \textit{i.e.}, $\mathbf{F} = (\mathbf{F}_{geo}, \mathbf{F}_{app})$. Recently, \cite{AlexYu2021PlenOctreesFR,AlexYu2022PlenoxelsRF} factorize the appearance via spherical harmonic representation:
	\begin{equation}\label{eq: sh}
		\boldsymbol{c} = \mathbf{F}_{sh}(\boldsymbol{x})\Gamma(\boldsymbol{d})+\boldsymbol{v},  
	\end{equation}
	where $\Gamma$ is the pre-defined basis function to produce the $\ell$-dimensional basis, $\mathbf{F}_{sh}$ computes the corresponding ${3\times \ell}$ coefficient matrix, and $\boldsymbol{v}$ is a fixed vector to normalize the colors. This form can greatly reduce the redundancy of the representation, and speed up the training and inference.  
	
	\noindent\textbf{Volume rendering.}
	A ray cast into the scene can be formulated as $\mathbf{r}(t) = \mathbf{o}+t\mathbf{d}$, where $\mathbf{o}$ and $\mathbf{d}$ are the origin and normalized direction of the ray,  and $t$ denotes the distance along the ray. The color of the ray with near and far bounds $t_1$ to $t_{T+1}$ is estimated by the volume rendering
	\begin{equation}\label{eq: alpha composition}
            \setlength\abovedisplayskip{2pt}
            \setlength\belowdisplayskip{2pt}
		{C}(\mathbf{r};\mathbf{F})=\sum\nolimits_{i=1}^T  w_i \boldsymbol{c}_i,
	\end{equation}
	\begin{equation}\label{eq: w}
		w_i=(1-e^{-\sigma_i(t_{i+1}-t_{i})})e^{-\sum_{i^{\prime}<i} \sigma_{i^{\prime}}(t_{i^{\prime}+1}-t_{i^{\prime}})},
	\end{equation}
	where $\sigma$ and $\boldsymbol{c}$ are predicted by the radiance field $\textbf{F}$. 
	
	\noindent\textbf{Lipschitz functions.}  
	Given two metric spaces $\mathcal{X}$ and $\mathcal{Y}$, a function $f : \mathcal{X} \rightarrow \mathcal{Y}$ is Lipschitz, or $K$-Lipschitz continuous if there exists $K \in \mathbb{R}^{+}$ satisfied that: 
	\begin{equation}
		\forall x_1, x_2 \in \mathcal{X},	\ d_{\mathcal{Y}}\left(f\left(x_1\right), f\left(x_2\right)\right) \leq K d_{\mathcal{X}}\left(x_1, x_2\right),
	\end{equation}
	where $d_{\mathcal{X}}$ and $d_{\mathcal{Y}}$ are metrics in $\mathcal{X}$ and $\mathcal{Y}$, respectively. $K$ is called the Lipschitz constant of $f$. In this paper, we define $d_{\mathcal{X}}$ and $d_{\mathcal{Y}} $ as (2-norm) Euclidean distances of vectors. 
	
	\section{Problem statement}\label{sec: Problem statement}
	\begin{figure}[t]
		\centering
		\includegraphics[width=0.99\linewidth]{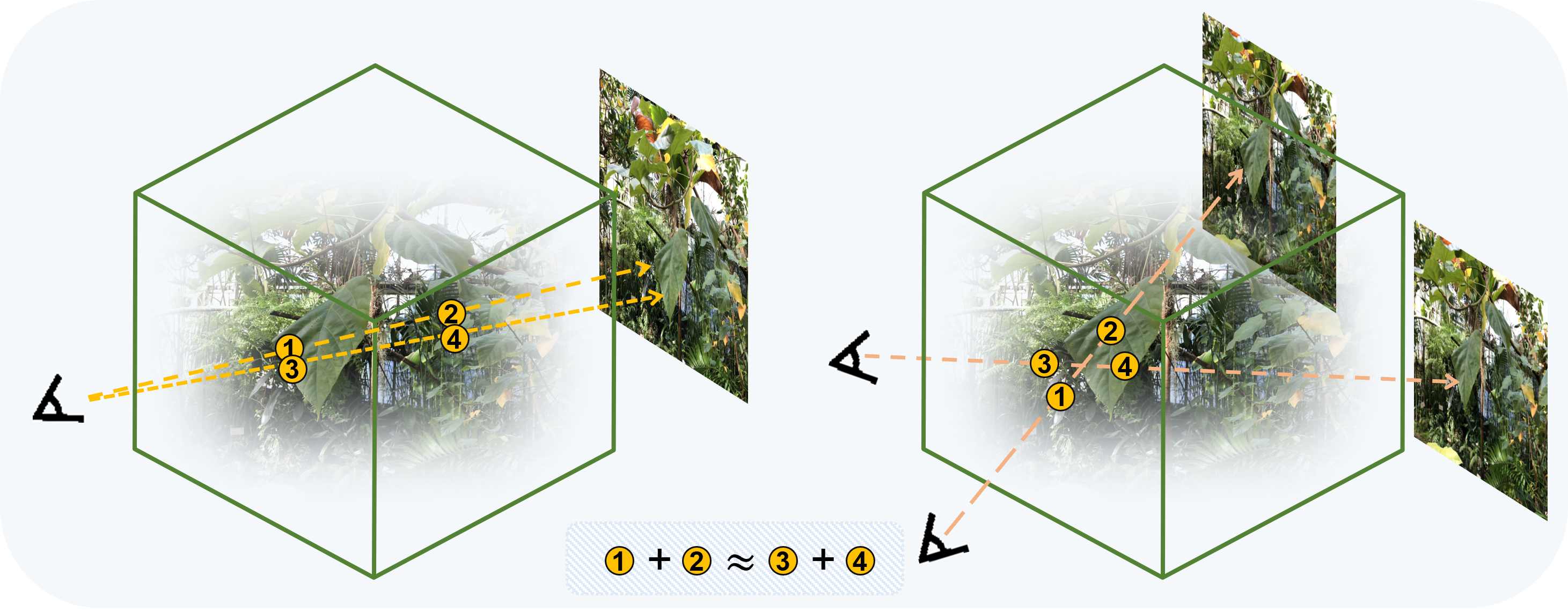}
		\caption{Illustrations of image structure (left) and cross-view consistency (right) from the perspective of volume rendering. For easy understanding, we placed images behind radiance fields (the cubes), and the two yellow circles on each ray denote the weighted colors sampled for volume rendering as Eq. \eqref{eq: alpha composition}. For each case, the volume rendering variance as Eq. \eqref{eq: vrr} should be a small value. 
		} 
		\vspace{-0.47cm}
		\label{fig: structure and consistency}
	\end{figure}   
	
	\noindent\textbf{Task and challenges.}
	Given a set of images $\{\mathbf{I}_i\}_{i=1}^{N}$ taken in a 3D scene $\mathcal{I}$ with known camera parameters, our goal is to synthesize novel photorealistic views with the similar color style to the reference image. 
	As discussed in Section \ref{sec:intro},  
	the main challenge is to protect \textit{photorealism} and \textit{consistency} while transferring the color style. More specifically,
	1) For photorealism, the image structure (\textit{e.g.},  edges and regions) needs to be well preserved after stylization. Without loss of generality, considering two nearby pixels at $\mathbf{p}_1$ and $\mathbf{p}_2$ in an image $\mathbf{I}$, we succinctly state that if $\left \| \mathbf{I}(\mathbf{p}_1) -\mathbf{I}(\mathbf{p}_2) \right \| <\epsilon$ ($\epsilon$ is a small value), the two belong to the same region; otherwise not. 2) For consistency, we suppose that one point in the scene can be observed by two adjacent views $\mathbf{I}_1$ and $\mathbf{I}_2$ at $\mathbf{p}_1$ and $\mathbf{p}_2$ (with a slight abuse of notation), respectively. Thus, the cross-view consistency should meet $\left \| \mathbf{I}_1(\mathbf{p}_1) -\mathbf{I}_2(\mathbf{p}_2) \right \| <\epsilon$. For this task, the above two relations between pixels in the stylized images should be consistent with that of the source images.
	
	\noindent\textbf{A novel perspective from volume rendering.} Due to the fact that the intensity of pixels is physically derived from the volume rendering of corresponding ray castings, we reinterpret them from the perspective of rays. Suppose that $\mathbf{r}_1$ and $\mathbf{r}_2$ intersect with the single/couple views at $\mathbf{p}_1$ and $\mathbf{p}_2$, referring to Eq. \eqref{eq: alpha composition}, we define the {volume rendering variance} ($vrr$) of rays $\mathbf{r}_1$ and $\mathbf{r}_2$ in the radiance field $\mathbf{F}$ by
	\begin{equation}\label{eq: vrr}
		\begin{small}
			\begin{aligned}
				vrr(\mathbf{r}_1,\mathbf{r}_2; \mathbf{F}) &= \left \| {C}(\mathbf{r}_1;\mathbf{F}) -{C}(\mathbf{r}_2;\mathbf{F})  \right \| \\ 
				&= \left \| \sum\limits_{i=1}^{T} w^{\mathbf{r}_1}_{i}\boldsymbol{c}^{\mathbf{r}_1}_i -\sum\limits_{i=1}^{T} w^{\mathbf{r}_2}_{i}\boldsymbol{c}^{\mathbf{r}_2}_i  \right \|, \\
			\end{aligned}  
		\end{small}
	\end{equation}	     
	where the superscripts denote the ray index. Importantly, the arrival of $vrr$ integrates the intricate relationships within (structure) and between (consistency) images into the variance over rays.
	Figure \ref{fig: structure and consistency} depicts a concise example. In this way,  if a radiance field $\mathbf{F}$ could represent the scene $\mathcal{I}$, the core challenge of this task can be streamlined to learn a stylized radiance field $\mathbf{F}'$ satisfying 
	\begin{equation}\label{eq: vrr relation}
		\begin{small}
			vrr(\mathbf{r}_1,\mathbf{r}_2; \mathbf{F}) < \epsilon \Rightarrow vrr(\mathbf{r}_1,\mathbf{r}_2; \mathbf{F}') < \epsilon',
		\end{small}
	\end{equation}	
	where $\epsilon'$ is a
	small value. In the following, we prove that the above demand can be fulfilled elegantly recurring to the Lipschitz mapping, thereby leading to decent stylization.

	\section{Methodology}
	In this section, we introduce a concise and flexible framework called LipRF to tackle photorealistic 3D scene stylization. LipRF first obtains the radiance field $\textbf{F}$ of the source scene (Sec. \ref{sec: Scene representation}). Based on the theoretical analysis  (Sec. \ref{sec: Theoretic}) of controlling $vrr$ with Lipschitz mappings, LipRF transforms the pre-trained radiance field with the Lipschitz MLP (Sec. \ref{sec: Lipschitz transform}) to reconstruct the views stylized by 2D PST. The gradual gradient aggregation (Sec. \ref{sec: GGA}) is elaborated to optimize LipRF in a cost-efficient way.
	
	\subsection{Scene representation via radiance field }\label{sec: Scene representation}
	
	In the first stage, a radiance field $\mathbf{F} = (\mathbf{F}_{geo}, \mathbf{F}_{app})$ with faithful geometry and appearance representation is trained to reconstruct the real scene $\mathcal{I}$. Following \cite{AlexYu2022PlenoxelsRF}, we adopt the reconstruction loss for training:
    \begin{equation}\label{eq: rec}
            \setlength\abovedisplayskip{1pt}
            \setlength\belowdisplayskip{1pt}
		\mathcal{L}_{rec}(\mathbf{F},\mathcal{I}) = \sum\limits_{i=1}^{m} \left\| {C}(\mathbf{r_i};\mathbf{F}) - {C}(\mathbf{r_i})   \right\|^2, 
	\end{equation}
	where $\{\mathbf{r}_i\}_{i=1}^{m}$ denotes the set of rays generated under the given camera parameters, ${C}(\mathbf{r_i};\mathbf{F})$ is the color estimated by volume rendering as in Eq. \eqref{eq: alpha composition}, and ${C}(\mathbf{r_i})$ is the groundtruth color of the corresponding image pixel in $\{\textbf{I}_{i}\}_{i=1}^{N}$. 
	
	We assume $\mathbf{F}_{geo}$ enables fully encoding the geometry of $\mathcal{I}$ after training. Since PST should not change the geometry of the source scene, we directly set $\mathbf{F}'_{geo} = \mathbf{F}_{geo}$. In this way, ${C}(\mathbf{r};\mathbf{F})$ and ${C}(\mathbf{r};\mathbf{F}')$ have the same rendering weights as in Eq. \eqref{eq: w} on the ray path.
	
	\subsection{Theoretic form of stylized radiance field}\label{sec: Theoretic}
	
	The classic 2D PST methods \cite{ErikReinhard2001ColorTB,Xiao2006ColorTI,Piti2007TheLM}, which simply transfer the color style by linear mappings of pixels, can well preserve the image structure.  In addition to linearity, we find another commonality of these methods that the corresponding Lipschitz constants are all of small values, \textit{e.g.}, usually less than 5 on the PST dataset \cite{FujunLuan2017DeepPS}. The fact indicates that Lipschitz property may also play an important role in maintaining structure of images. Prompted by this underlying relationship, we prove that the Lipschitz-constrained linear mapping of $\textbf{F}_{app}$ is indeed an optimal form for holding Cond. \eqref{eq: vrr relation}:
    
	\begin{proposition}\label{proposition 1}
		Considering  $f(\boldsymbol{c}) = \boldsymbol{A}\boldsymbol{c}+\boldsymbol{b}$, $\boldsymbol{A}\in \mathbb{R}^{3\times3}$, $\boldsymbol{b}\in \mathbb{R}^{3\times1}$, if $\mathbf{F}_{app}' = f\circ\mathbf{F}_{app}$, $\sum_{i=1}^{T}w_{i} = 1$ and  $vrr(\mathbf{r}_1,\mathbf{r}_2; \mathbf{F})<\epsilon$, we have  $vrr(\mathbf{r}_1,\mathbf{r}_2; \mathbf{F}')<K\epsilon$, where  $K =\left\| \boldsymbol{A} \right\|_2$ is the Lipschitz constant of $f$\footnote{\label{footnote1} Proof is provided in the supplementary materials.}. 
	\end{proposition}
 
	Since the unique assumption of $\sum_{i=1}^{T}w_{i} = 1$ is in accordance with many radiance field models \cite{Mildenhall2020NeRFRS,AlexYu2021PlenOctreesFR,AlexYu2022PlenoxelsRF} in practice, Prop. \ref{proposition 1} provides an ideal way to establish $\textbf{F}'$, where the variational bound $\epsilon'$ in Cond. \eqref{eq: vrr relation} is influenced by the Lipschitz constant of transform.  Nonetheless, the limited expressivity of linear mapping definitely affects dramatic style effects, \textit{e.g.}, the color variation in Figure \ref{fig: motivation} by the deep PST method.  We propose to mitigate the deficiency while maintaining Cond. \eqref{eq: vrr relation} by means of the Lipschitz MLP:
	\begin{proposition}\label{proposition 2}
		Considering $f=f_l \circ\cdots \circ f_1$,  $f_j(x) = \boldsymbol{A}_jx+\boldsymbol{b}$ if $j = l$ and  $\sigma(\boldsymbol{A}_{j}x)$ otherwise, where $\sigma=\max(0,x)$. If $\mathbf{F}_{app}' = f\circ\mathbf{F}_{app}$, $\sum_{i=1}^{T}w_{i} = 1$ and $\max_{i=1,\dots,T}\left\|  w^{\mathbf{r}_1}_{i}\boldsymbol{c}^{\mathbf{r}_1}_i - w^{\mathbf{r}_2}_{i}\boldsymbol{c}^{\mathbf{r}_2}_i\right\| < \epsilon/T$,  we have $vrr(\mathbf{r}_1,\mathbf{r}_2; \mathbf{F}')<K\epsilon$, where $K = \Pi^l_{i=1}\left\| \boldsymbol{A}_{i} \right\|_{2} $ is the Lipschitz constant of $f$\footref{footnote1}. 
	\end{proposition}
	This proposition further necessitates the vanishing of $\left\|  w^{\mathbf{r}_1}_{i}\boldsymbol{c}^{\mathbf{r}_1}_i - w^{\mathbf{r}_2}_{i}\boldsymbol{c}^{\mathbf{r}_2}_i\right\|$, which is valid when the rays are quite adjacent. Despite the premise seems to hamper the utilization of Lipschitz MLP, it can be greatly loosened for the close relation between Lipschitz MLP and linear mapping in Prop.~\ref{proposition 1}: 1)  The above Lipschitz MLP as a piece-wise linear function behaves the same as linear mapping in a local space \cite{Montfar2014OnTN}. 2) The Lipschitz MLP with strict constraint of Lipschitz condition and gradient norm  approximates to the linear mapping \cite{CemAnil2018SortingOL}. These properties encourage to form the stylized radiance field as:
	\begin{equation}\label{eq: theoretic form }
		\mathbf{F}' = (\mathbf{F}_{geo}, f\circ\mathbf{F}_{app}), \  f \  \text{is}\  K\text{-Lipschitz MLP}.
	\end{equation}
	
	Regarding the training complexity of $f$, generally if there are $n$ values sampled for each variable of the position and direction, $f$ needs to take a large amount of parameters with high computational costs for predicting $O(n^5)$ colors correctly. Fortunately, we find a nice property\footref{footnote1} in Eq. \eqref{eq: sh}:
	\begin{equation}
		\boldsymbol{A}\mathbf{F}_{app}({\boldsymbol{x}, \boldsymbol{d}})+\boldsymbol{b} \Leftrightarrow 
		\boldsymbol{A}\mathbf{F}_{sh}({\boldsymbol{x}})+2\sqrt{\pi}[\boldsymbol{A}\boldsymbol{v}+\boldsymbol{b}-\boldsymbol{v},\boldsymbol{0}],
	\end{equation}
	which allows to exchange the linear mappings of appearance representation and spherical harmonic representation. Therefore, $\mathbf{F}_{app}' = f\circ\mathbf{F}_{app}$ can be pared down to  $\mathbf{F}'_{sh} = f\circ\mathbf{F}_{sh}$, while not violating the above propositions. By doing so, it is feasible to design $f$ as a lightweight model to handle the $O(n^3)$ spherical harmonic coefficients.

	\subsection{Lipschitz transformation of radiance field }\label{sec: Lipschitz transform}
	Based on the above analysis, the second step of LipRF is to transform the radiance field $\mathbf{F}$ with Lipschitz MLP $f$ as stylized radiance field $\mathbf{F}'$.  
	Here $f$ is composed of linear and activation layers. It receives and updates the flattened spheric harmonic coefficients.  We also input the 3D position for spatial inductive bias, namely $\mathbf{F}_{sh}'(\boldsymbol{x}) = f(\mathbf{F}_{sh}(\boldsymbol{x}),\boldsymbol{x})$.  The training objective is
	\begin{equation}\label{eq: total loss}
		\mathop{\min}\nolimits_{f} \mathcal{L}_{rec}(\mathbf{F}', \mathcal{S}) + \lambda \mathcal{L}_{Lip}(f).
	\end{equation}
	$\mathcal{S}$ denotes the scene consisting of  $\{pst(\mathbf{I}_{i})\}_{i=1}^{N}$, where $pst$ is an arbitrary 2D PST method.  $\mathcal{L}_{rec}(\mathbf{F}', \mathcal{S})$ is the reconstruction loss of the same form as Eq. \eqref{eq: rec}, and $\mathcal{L}_{Lip}$ is the proposed adaptive Lipschitz regularization to adjust the Lipschitz constant of $f$, and $\lambda$ is the balance weight.  
	
	\noindent\textbf{Lower bound of Lipschitz constant.} In practice, it is difficult to determine the optimal Lipschitz constant $K$ of the network for different scenes and reference images. Large values will invalidate the constraint of $vrr$, while small values may result in over-constrained conditions for the color style transfer. To address this problem, we first estimate a value $K_{est}$ as the lower bound
	of $K$. Here we adopt a widely used color transfer method, \textit{i.e.}, Monge-Kantorovitch linear \cite{Piti2007TheLM} (MKL) mapping, to compute the transfer matrix between $\mathcal{I}$ and $\mathcal{S}$:
	\begin{equation}\label{eq: MKL}
		\boldsymbol{M}=\boldsymbol{\Sigma}_{\mathcal{I}}^{-1 / 2}\left(\boldsymbol{\Sigma}_\mathcal{I}^{1 / 2} \boldsymbol{\Sigma}_\mathcal{S} \boldsymbol{\Sigma}_\mathcal{I}^{1 / 2}\right)^{1 / 2} \boldsymbol{\Sigma}_\mathcal{I}^{-1 / 2},
	\end{equation}
	where $\boldsymbol{\Sigma}_{\mathcal{I}}$ and $\boldsymbol{\Sigma}_{\mathcal{S}}$ are the covariance matrices of pixel colors in $\{\mathbf{I}_{i}\}_{i=1}^{N}$ and $\{pst(\mathbf{I}_{i})\}_{i=1}^{N}$, respectively. We specify $K_{est} = \left\| \boldsymbol{M} \right\|_2 $ so that $f$ outperforms the linear mapping.

	\noindent\textbf{Adaptive Lipschitz regularization.} Since the Lipschitz constant of network is affected by that of each linear layer (see Prop. \ref{proposition 2}), it is hard to optimize $K$ directly.  We reform
	\begin{equation}
		\boldsymbol{A}_i = \text{squareplus}(K_i, b) \boldsymbol{W}_i/\left\|\boldsymbol{W}_i\right\|_2,
	\end{equation}
	where $\boldsymbol{W}_i$ and $K_i$ are the parameters to optimize for the $i$-th linear layer. Here, $\text{squareplus}(x, b)=\frac{1}{2}\left(x+\sqrt{x^2+b}\right)$ \cite{Barron2021SquareplusAS} is similar to ReLU, but always produces positive values to prevent the norm of $\boldsymbol{A}_{i}$ from vanishing. We follow \cite{Miyato2018SpectralNF} to fast approximate $\left\|\boldsymbol{W}_i\right\|_2$ by one-step power iteration. On this basis, the regularization is defined as 
	\begin{equation}\label{eq: llip}
         \setlength\abovedisplayskip{1pt}
         \setlength\belowdisplayskip{1pt}
		\mathcal{L}_{Lip} = \sum_{i = 1}^{l} \text{squareplus}(K_i-\sqrt[l]{K_{est}}, b). 
	\end{equation}
	The Lipschitz constant of network is optimized by constraining the norm of each linear layer to the geometric mean value of $K_{est}$, and thus $\mathcal{L}_{Lip}$ can softly control the gap between $K$ and $K_{est}$ during training. 
	\setlength{\textfloatsep}{0.4cm}
	\begin{algorithm}[t]
		\SetAlgoLined
		\begin{footnotesize}
			\begin{scriptsize}
				\PyComment{f - Lipschitz MLP; Opt - optimizer of f; sh - [B,$\ell$] spheric harmonic coefficients of $\mathbf{F}$; sigma - [B,1] density of $\mathbf{F}$; rs - rays of a view; C - groundtruth; idx - indexes for splitting batches} \\
			\end{scriptsize}
			\PyCode{Opt.zero\_grad()} \\
			\PyComment{0. feed forward} \\
			\PyCode{With torch.no\_grad():} \\
			\PyCode{\ \ \ \ sh\_t = cat([f(sh[i:j]) for i,j in idx])} \\
			\PyCode{\ \ \ \ C\_hat = volume\_render(rs, sigma, sh\_t)} \\ 
			\PyComment{1. backward from rec loss to image} \\
			\PyCode{rec\_loss(C\_hat, C).backward()} \\
			\PyComment{2. backward from image to sh\_t} \\
			\PyCode{p = volume\_render(rs, sigma, sh\_t)} \\
			\PyCode{p.backward(grad = C\_hat.grad)} \\
			\PyComment{3. backward from sh\_t and Lip loss to f} \\
			\PyCode{for i, j in idx:} \\
			\PyCode{\ \ \ \ p = f(sh[i:j])} \\
			\PyCode{\ \ \ \ p.backward(grad = sh\_t[i:j].grad/B)} \\
			\PyCode{(lambda * Lip\_loss(f)).backward()} \\
			\PyCode{Opt.step()} \\
		\end{footnotesize}
		\caption{PyTorch-style pseudocode for GGA}
		\label{algo: ggp}
	\end{algorithm}
	
	\subsection{Optimization by gradual gradient aggregation}\label{sec: GGA}
	For Plenoxels \cite{AlexYu2022PlenoxelsRF}, Lipschitz MLP needs to transform the spherical harmonic coefficients on all vertices at each iteration due to the structure of voxel grid. Therefore, training LipRF will cost a massive GPU memory in both feed forward and back propagation.  The intuitive strategy that optimizes a sparse set of rays \cite{PeiZeChiang2021Stylizing3S,ThuNguyenPhuoc2022SNeRFSN} at once will cause considerable redundancy, since the majority of vertices are not selected.  \cite{Zhang2022ARFAR} proposes to defer the back propagation for removing useless gradient cache, but it still fails on LipRF due to the huge amount of inputs for Lipschitz MLP. 
	
	To increase training efficiency, we propose the gradual gradient aggregation (GGA) detailed in Algorithm \ref{algo: ggp}. First, GGA does not construct the computation graphs during the forward process to reduce memory footprint. Then, the GGA gradually propagates the gradient after redoing each forward step. Finally, the gradient of Lipschitz MLP is aggregated in a batch-wise way. Since the forward process costs much less time than the backward propagation, GGA enables a fast training speed, and tractable memory footprint by adjusting the number of batches. 
		\begin{figure*}[t]
		\centering
		\includegraphics[width=1\linewidth]{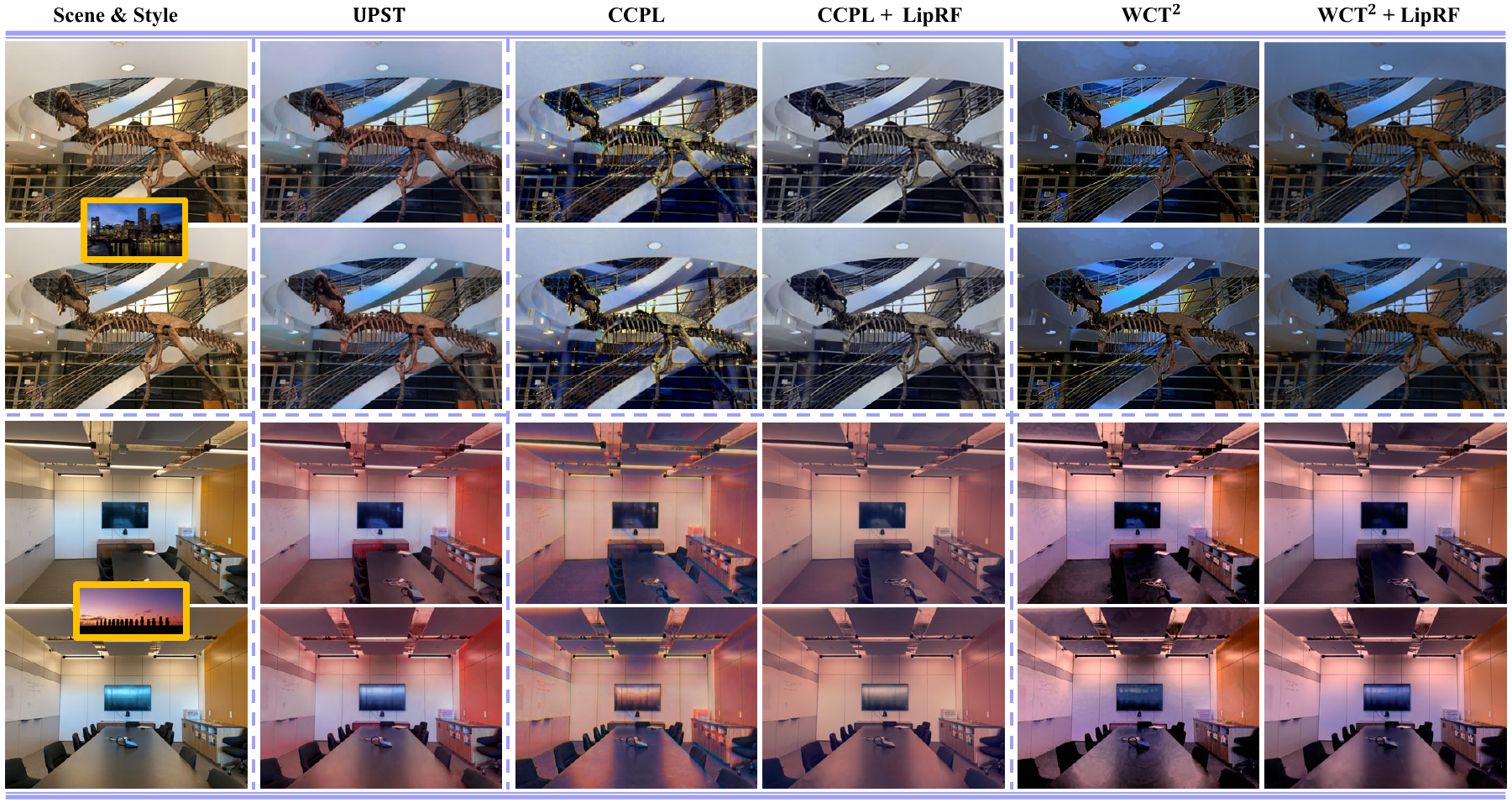}
		\caption{Comparison on “trex” and “room” scenes from LLFF dataset \cite{BenMildenhall2019LocalLF}, where the image resolution is $1008 \times 756$. For the two scenes, LipRF is learned with the prior knowledge derived from the stylized results of $\text{WCT}^2$ \cite{Piti2007TheLM} and CCPL \cite{Piti2007TheLM}.}
		\label{fig: llff_comp}
		\vspace{-0.3cm}
	\end{figure*}
		\begin{figure}[t]
		\centering
		\includegraphics[width=0.99\linewidth]{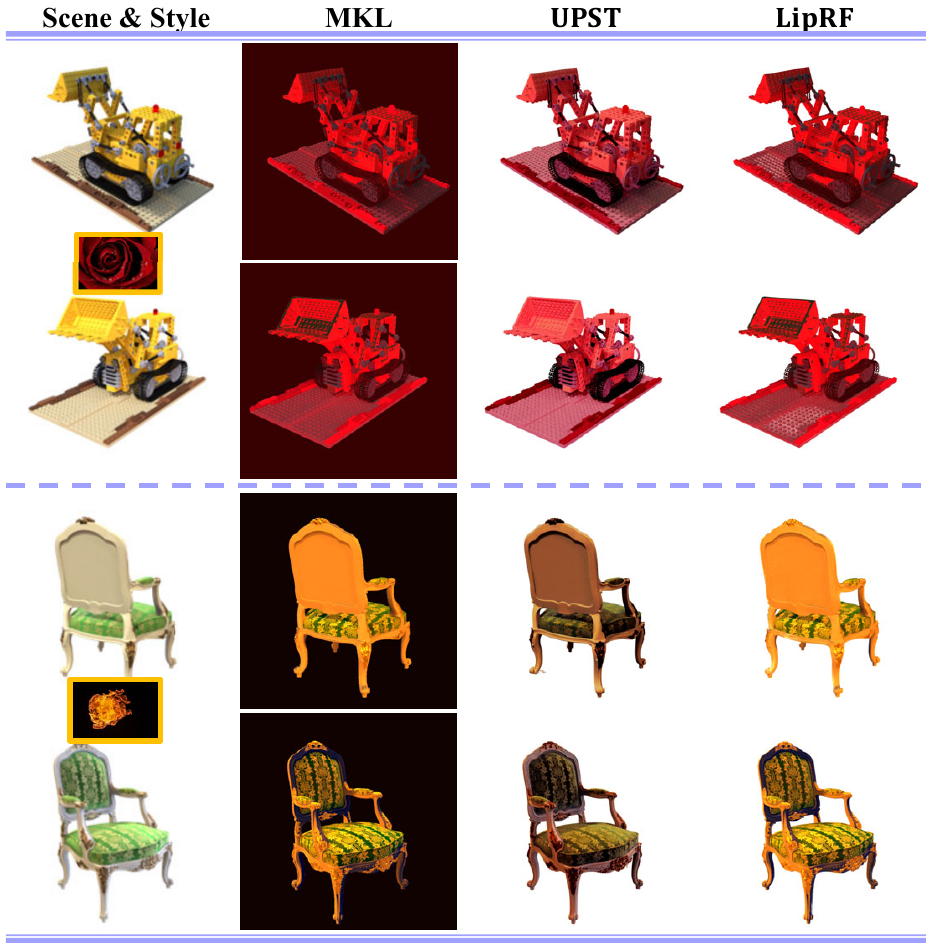}
		\caption{Comparison on NeRF-synthetic dataset \cite{Mildenhall2020NeRFRS}, where the image resolution is $800\times 800$.  LipRF is learned with the prior knowledge derived from the stylized scene of MKL \cite{Piti2007TheLM}. }
		\label{fig: syn_comp}
	\end{figure} 
	\section{Experiments}

	\noindent\textbf{Settings.} 
	The overall architecture of LipRF is implemented based on the official code of Plenoxels\footnote{\url{https://github.com/sxyu/svox2}} \cite{AlexYu2022PlenoxelsRF}. The first stage of training radiance field is the same as in Plenoxels. In the second stage of learning Lipschitz network, the MLP has 5 linear-activation layers and 1 linear output layer. The middle layer has 64 neural units. Note that taking the (1-Lipschitz) sinusoidal function \cite{sitzmann2019siren} as activation can lead to better results than ReLU or LeakyReLU sometimes. We use the Adam optimizer \cite{Kingma2015AdamAM}, where $\beta_1 = 0.9$, $\beta_2 = 0.999$, and the learning rate is reduced from $10^{-2}$ to $10^{-4}$ by cosine annealing \cite{Loshchilov2017SGDRSG}.  We set $\lambda = 2\times 10^{-4}$ in Eq. \eqref{eq: total loss} and $b = 10^{-12}$ in $\text{squareplus}$. Unless otherwise stated, the optimization process runs $300$ epochs in total and takes no more than 7 minutes on a single NVIDIA RTX 3090.  
	
	\noindent\textbf{Datasets.}
	We conduct qualitative and quantitative evaluations to verify LipRF on various scenes in multiple datasets, including \textit{NeRF-synthetic} dataset \cite{Mildenhall2020NeRFRS} of synthetic scenes,  \textit{LLFF}  \cite{BenMildenhall2019LocalLF} dataset of real forward-facing scenes, and some scenes from \textit{Tanks and Temples} dataset \cite{AKnapitsch2017TanksAT} of real $360^{\circ}$ scenes and the multi-view stereo \textit{DTU} dataset \cite{Jensen2014LargeSM}.  The style images are derived from the PST dataset \cite{FujunLuan2017DeepPS}. Due to space limitations, we present part of them in the paper. Please refer to the supplementary materials for more results.  
	
	\noindent\textbf{Baselines.}
	We leverage three existing 2D PST approaches as baselines. In particular,  MKL \cite{Piti2007TheLM} is a classic and widely used color transfer method. $\text{WCT}^{2}$ \cite{JaejunYoo2019PhotorealisticST} is recognized as a typical baseline of 2D PST methods, which is stable for videos or high-resolution images.  $\text{CCPL}$ \cite{Wu2022CCPLCC} is a recent state-of-the-art 2D PST method. 
	Besides, we include a concurrent work UPST \cite{Chen2022UPSTNeRFUP} for comparison, which promotes 2D PST to tackle the same task of photorealistic 3D stylization.

	\subsection{Qualitative results}
	
	\noindent\textbf{NeRF-synthetic dataset.}
	Figure \ref{fig: syn_comp} summarizes the comparison between our LipRF and existing 2D/3D PST methods. Considering that the scenes in synthetic dataset \cite{Mildenhall2020NeRFRS} are constructed with simple texture, we choose the classic
	MKL \cite{Piti2007TheLM} as the only 2D PST baseline. Although MKL can transfer the colors faithfully, it fails to distinguish between the foreground and background, resulting in a drastic change in the background pixel colors. It is worth noting that this is indeed a common problem existing in all 2D PST methods when being simply applied to 3D scene. Moreover, the stylized results of UPST \cite{Chen2022UPSTNeRFUP} show some disparity with the reference color style in vision. In contrast, our approach manages to integrate the advantages of radiance field with MKL. Since the densities of background regions are all 0 in the radiance field, they do not contribute to the volume rendering, encouraging the rendered colors to be consistent with the source scene background.
	
	\noindent\textbf{LLFF dataset.}
	We further illustrate the comparisons on LLFF dataset in Figure \ref{fig: llff_comp}. As shown in the results, the state-of-the-art 2D PST approach (CCPL \cite{Wu2022CCPLCC})  has two main downsides: First, there are a lot of noises in the stylized scene, like contrasting spots and variegation. Second, the stylized results are obviously blurred compared to the source images. This is mainly because it is non-trivial to apply the Encoder-Decoder architecture in CCPL trained on COCO dataset \cite{TsungYiLin2014MicrosoftCC} for the high-resolution inputs of scene images. $\text{WCT}^2$ \cite{Wu2022CCPLCC} performs better than CCPL in structure preservation due to its wavelet module, but the inevitable noises and the intense edges (\textit{e.g.}, the bone boundary of trex) also cause disharmony. By training a high-resolution 2D PST network, UPST \cite{Chen2022UPSTNeRFUP} is able to preserve the
	structure well. However, the color style of stylized scene is different from the reference. Instead, LipRF completely eliminates the limitations of CCPL and $\text{WCT}^2$, and the stylized results have clear advantages in photorealism and color style.
	
		\begin{figure}[t]
		\centering
		\includegraphics[width=0.99\linewidth]{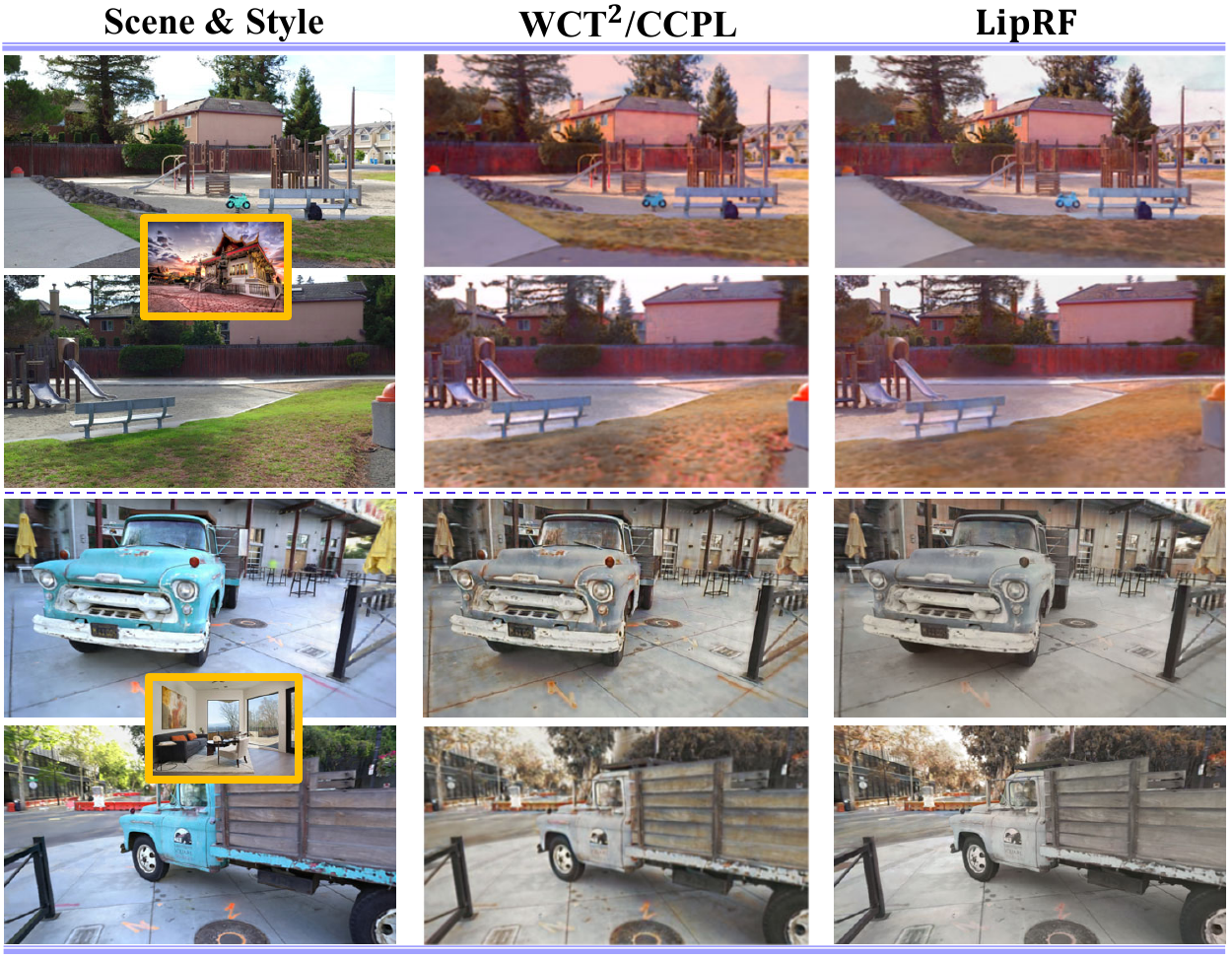}
		\caption{Comparison on Tanks and Temples dataset \cite{AKnapitsch2017TanksAT}.  For two scenes, LipRF is learned with the prior knowledge derived from the stylized results of $\text{WCT}^2$ (first) \cite{Piti2007TheLM} and CCPL \cite{Piti2007TheLM} (second).}
		\label{fig: tnt_comp}
		\vspace{-0.3cm}
	\end{figure}
	
	\begin{figure}[t]
		\centering
		\includegraphics[width=0.99\linewidth]{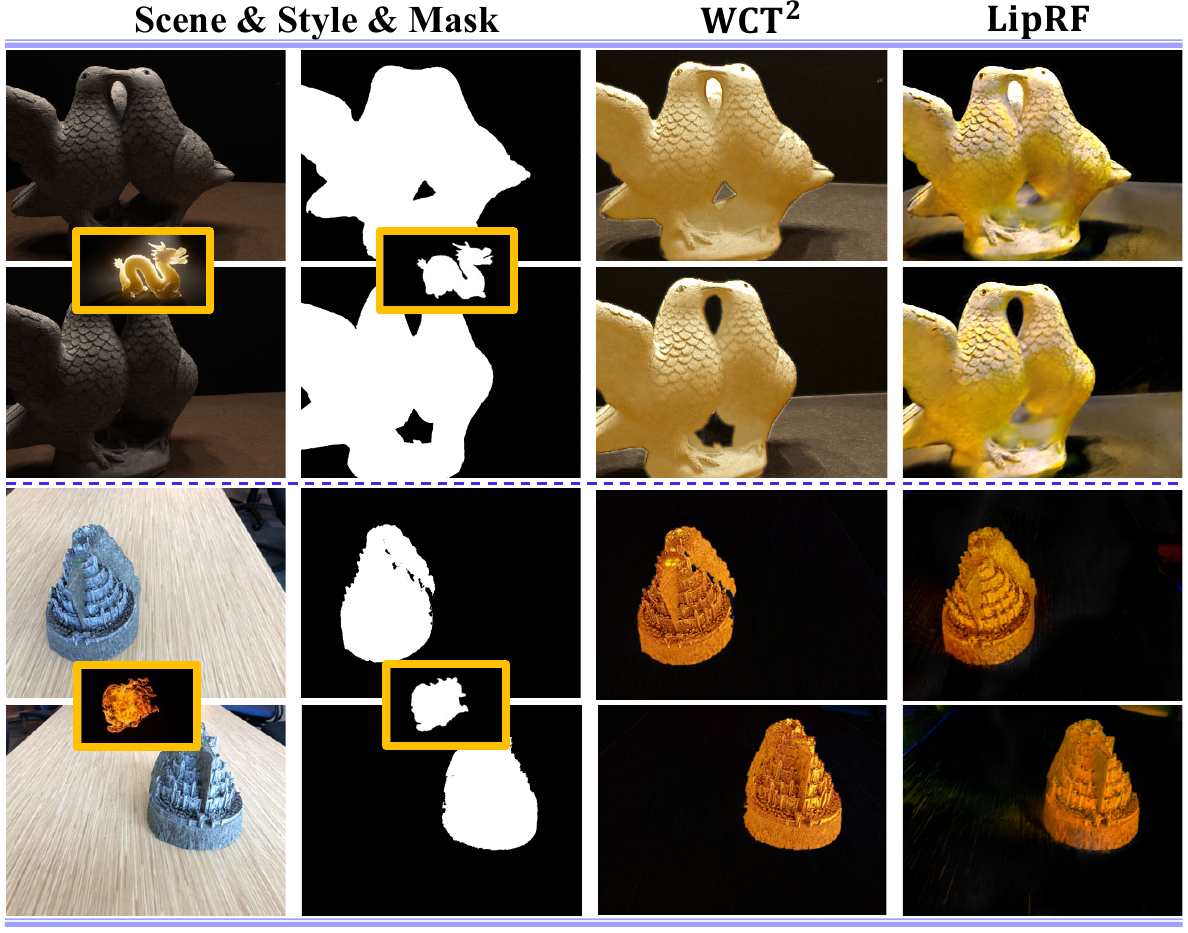}
		\caption{Object appearance editing. The semantic masks of style image and scene are provided to guide the style transfer.  }
		\label{fig: object_edit}
		\vspace{-0.3cm}
	\end{figure}

	\noindent\textbf{Tanks and Temples dataset.}
	Since UPST \cite{Chen2022UPSTNeRFUP} does not support this dataset, Figure \ref{fig: tnt_comp} depicts the comparison between LipRF and 2D PST methods. Similarly, 2D PST methods result in cross-view inconsistency like the sky color on the playground, and strong noises like the messy colors on car and ground. After transforming radiance field with Lipschitz Network, our LipRF manages to alleviate these issues and obtains high-quality stylized results.
	
	\noindent\textbf{Object appearance editing.}
	Another practical application of PST is object appearance editing \cite{FujunLuan2017DeepPS}, which leverages additional semantic masks of objects to guide style transfer. However, it is difficult to precisely annotate the object masks of each view. In Figure \ref{fig: object_edit}, the first scene comes from the DTU dataset \cite{Jensen2014LargeSM}, and the provided inaccurate masks result in the inconsistency of stylized images.  The same problem is also observed for the
	second scene taken from LLFF \cite{BenMildenhall2019LocalLF}  with automatic annotation. In contrast, LipRF obtains photorealistic and consistent results.
	
	\noindent\textbf{Style interpolation.}
	Once the training of LipRF is completed, we can interpolate the source and stylized radiance fields to obtain
	$\textbf{F}^{\alpha}_{sh} = \alpha\textbf{F}'_{sh}+(1-\alpha)\textbf{F}_{sh}$ by adjusting the factor $\alpha \in [0,1]$ (see Figure \ref{fig: interp}). This provides an efficient way to control or serialize the style change of scenes, saving much time compared with the interpolation of image pixels.

	\subsection{Quantitative results}

	\noindent\textbf{Consistency.}
	We take the temporal consistency \cite{Lai-ECCV-2018} as the metric for evaluating the cross-view consistency. Specifically, given two view synthesis $x_{i}$ and $x_{j}$ of a scene, the temporal consistency is computed by
		\begin{equation}
			{TC}\left(x_i, x_j\right)=\frac{1}{\left|\mathcal{O}_{i, j}\right|}\left\|\mathcal{O}_{i, j} \mathcal{W}_{i, j}\left(x_i\right)-\mathcal{O}_{i, j} x_j\right\|^2,
		\end{equation}
	where $\mathcal{W}_{i,j}$ warps $x_i$ to $x_j$ according to the optical flow estimated by RAFT \cite{Teed2020RAFTRA}, and mask $\mathcal{O}_{i,j}$ labels non-occluded pixels \cite{Ruder2016ArtisticST} in $x_i$ and $x_j$. $|\mathcal{O}_{i,j}|$ is the sum of tensor items. We further convert $TC$ to the readable PSNR form
		\begin{equation}\label{eq: psnr tc}
			{TC}_{psnr}(x_i, x_j) = -10 * \log_{10}(TC\left(x_i, x_j\right)).
		\end{equation}
	\noindent The metric is conducted on 8 scenes of LLFF dataset \cite{BenMildenhall2019LocalLF}, and each scene will be stylized based on 4 images from PST dataset \cite{FujunLuan2017DeepPS}. The number of evaluated views is the same as that of training views. We report the short temporal consistency taking ($x_{i}, x_{i+1}$) as the input pair, and the long temporal consistency taking ($x_{i}, x_{i+5}$) as input pair. Table \ref{table: short temporary} details the average results and we have three main observations. The first is that LipRF is the only method that manifests similar or advanced properties over linear method MKL \cite{Piti2007TheLM}. Second, LipRF stably promotes the consistency of 2D PST, thereby upgrading 2D PST to adapt for the 3D scene. Finally, LipRF does benefit from the improvement of 2D PST method, as evidenced by exhibiting better consistency of $\text{WCT}^2$ with LipRF against CCPL with LipRF.
	
    \noindent\textbf{User study.}
	We further conduct subjective evaluation to compare the style effects. We operate LipRF with $\text{WCT}^2$, UPST and MKL on 16 pairs of scenes and style images,  then invite 35 volunteers to comprehensively assess and vote for their favorite rendered videos. Finally, LipRF obtains $79\%$ preference that surpasses the $11\%$ of UPST and $10\%$ of MKL, proving the impressive effects of LipRF.  
	
	\begin{figure}[t]
		\centering
		\includegraphics[width=1.0\linewidth]{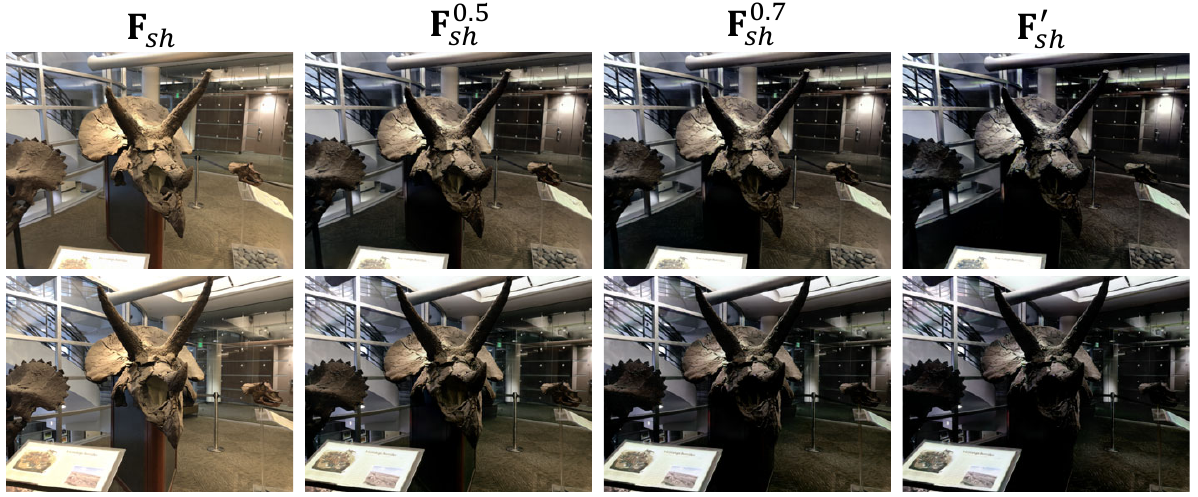}
         \vspace{-0.5cm}
		\caption{Style interpolation within radiance fields. }
		\label{fig: interp}
		\vspace{-0.2cm}
	\end{figure}
	\begin{table}[t]
		\small
		\setlength\tabcolsep{1.5pt}
		\begin{tabular}{c|c|c|c|c|c|c|c|c|c}
			\hline\hline
			& fern & flower & fortress & horns & leaves & orchids & room & trex & Avg. \\ \hline
			$\text{MKL}$&    28.4  &     32.9   &     \textbf{34.2}     &   31.1    &     \textbf{26.2}  &   27.0      &   28.8   &   27.7   &  30.0 \\ \hline
			$\text{WCT}^2$&    24.0  &     28.5   &     24.1     &   27.9    &     22.7   &   24.3      &   25.9   &   24.7   &  25.3  \\
			
			w/ Lip&   \textbf{29.5}   &   \textbf{34.1}     &    {33.4}      &    \textbf{32.6}  &    25.8    &   \textbf{28.0}      &   \textbf{29.4}   &   \textbf{28.7}   &   \textbf{30.3}  \\ 	\hline
			
			CCPL&   22.0   &   23.4     &     22.5      &    25.1   &   20.5     &    21.1    &   23.9   &  24.7    &   22.9   \\
			
			w/ Lip&   26.7   &     30.7   &     31.0    &   30.9    &   \textbf{26.2}     &     262    &    28.8  &   27.4   &  28.5  \\ 	\hline
			UPST&   27.6   &   33.4     &     31.0     &   30.5    &   \textbf{26.2}     &    26.7     &   30.4   &  27.7    & 28.3 \\   
			\hline\hline
			
			\hline\hline
			$\text{MKL}$&    25.6  &     26.3  &    \textbf{28.0}     &   26.3    &     \textbf{22.9}   &   23.6      &   24.6   &   23.8   &  25.1 \\ \hline
			$\text{WCT}^2$&    22.1  &     23.7   &    19.1     &   24.4    &     20.0   &   21.5      &   22.2   &   21.4   &  21.8  \\
			w/ Lip&   \textbf{26.7}   &   \textbf{28.0}     &    {27.9}      &    \textbf{27.8}  &    {22.5}    &   \textbf{24.6}      &   {25.3}   &   \textbf{24.3}   &   \textbf{25.9}  \\ 	\hline
			
			CCPL&   20.5   &   19.8     &    18.3       &    22.5   &  18.7     &    18.9    &   21.5   &  22.4    &    20.3   \\
			
			w/ Lip&   24.1   &    24.6   &    24.1    &   26.3    &   22.9     &     22.9    &    24.7  &  23.6   & 24.2  \\ 	\hline
			UPST&   24.8   &   26.6     &     22.7     &   26.0    &   21.2    &    23.6     &   \textbf{25.7}   &  23.7   & 24.3 \\   
			\hline\hline
		\end{tabular}
		\vspace{-0.3cm}
		\caption{Comparisons of short (upper) and long  (lower) temporary consistency on LLFF dataset. The higher the value, the better the consistency. “w/ Lip” means coupling LipRF with the 2D PST method in the block. The last column shows the average value.}
		\label{table: short temporary}
		\vspace{-1em}
	\end{table}
	\begin{figure}[t]
		\centering
		\includegraphics[width=0.99\linewidth]{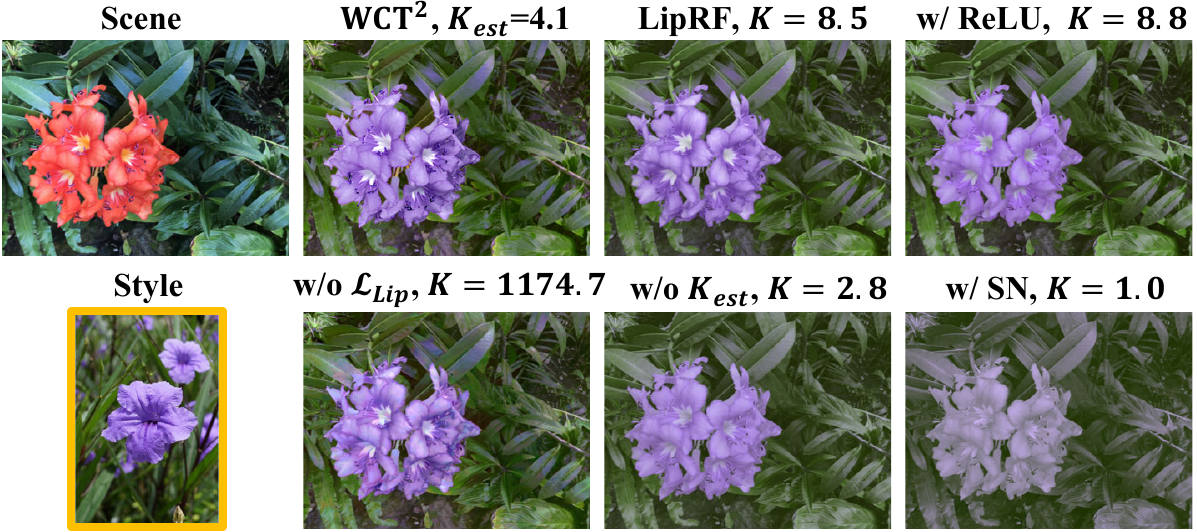}
         \vspace{-0.1cm}
		\caption{The results with different Lipschitz constants. $K_{est}$ is the estimated value in Eq. \eqref{eq: MKL} . “w/ ReLU” replaces the sinusoidal activation of LipRF with ReLU. “w/o $K_{est}$” sets the lower bound to zero. “w/ SN” replaces the adaptive regularization with the spectral normalization that forces the Lipschitz constant to 1.  }
		\label{fig: ablation}
  	\vspace{-1em}
	\end{figure}

	\subsection{Ablation study}
	\noindent\textbf{Alternative activations.}
	The activation is essential for constructing Lipschitz networks, and the commonly used ReLU (or LeakyReLU), Sigmoid, Sine and Tanh are all Lipschitz continuous. In the experiments, we found that Sigmoid and Tanh do not work due to the vanishing gradients sometimes. In the remaining alternatives, the Sine-based LipRF can obtain more faithful color style compared with the ReLU-based one.  Figure \ref{fig: ablation} visualizes the comparisons.
	
	\noindent\textbf{Lipschitz regularization.} 
	Next, we study the impact of the proposed adaptive Lipschitz regularization for LipRF. As shown in Figure \ref{fig: ablation}, LipRF will create noisy details when removing the regularization. Meanwhile, the Lipschitz constant surges from 8.5 to 1174.7. Furthermore, when removing or setting $K_{est}$ to zero, the Lipschitz constant of LipRF will decrease dramatically from 8.5 to 2.8 after 300 epochs. Both qualitative results validate our strategy. We also compare our strategy with spectral normalization \cite{Miyato2018SpectralNF}, and the results indicate that the lower the Lipschitz constant, the weaker the color change, and the visually smoother the image. The comparison again proves the merits of LipRF.

	\section{Conclusion}
	In this paper, we present a novel and well-motivated framework namely LipRF to tackle photorealistic 3D scene stylization, where the main challenge is the preservation of photorealism and consistency during stylization. To this end, we first show that the two objectives can be simplified
into maintaining volume rendering variance before and after stylization. Then, we prove that a Lipschitz MLP enables controlling the variance. Third, we propose an adaptive regularization to constrain the lower bound of Lipschitz constant and introduce the gradual gradient aggregation to optimize LipRF in a cost-efficient manner. Extensive experiments shows the versatility of LipRF. The main limitation is that, LipRF relies on the pre-trained radiance field. Once the radiance
field cannot reconstruct the scene well, the learning of
LipRF would be failed. We think this problem
can be solved with the development of radiance fields.

    \noindent\textbf{Acknowledgements.} This paper is supported by the National key research and development program of China (2021YFA1000403), and the National Natural Science Foundation of China (Nos. U19B2040, 11991022).
   
	{\small
		\bibliographystyle{ieee_fullname}
		\bibliography{citation}
	}
	
\end{document}


	\title{Transforming  Radiance Field with Lipschitz Network for 
		
		Photorealistic 3D Scene Stylization
  
  \textemdash\textemdash CVPR 2023 Supplementary Material}
	
	 \author{%
 	Zicheng Zhang\textsuperscript{1}
	\quad
	Yinglu Liu\textsuperscript{2}
	\quad
	Congying Han\textsuperscript{1}
	\quad
        Yingwei Pan\textsuperscript{2}
        \quad
	Tiande Guo\textsuperscript{1}
	\quad
	Ting Yao\textsuperscript{2}
	\quad
	\\ 
	\textsuperscript{1}University of Chinese Academy of Sciences
	\quad
	\textsuperscript{2}JD AI Research
	\quad
	\vspace{.5em} 
	\\
	\tt\small zhangzicheng19@mails.ucas.ac.cn
	\quad
	liuyinglu1@jd.com
	\quad
	hancy@ucas.ac.cn
	\\
	\tt\small
     panyw.ustc@gmail.com
	\quad
	tdguo@ucas.ac.cn
	\quad
	tingyao.ustc@gmail.com
}
	\maketitle
	\appendix

	\section{Proofs}
    \begin{proposition}\label{proposition 1}
		Considering  $f(\boldsymbol{c}) = \boldsymbol{A}\boldsymbol{c}+\boldsymbol{b}$, $\boldsymbol{A}\in \mathbb{R}^{3\times3}$, $\boldsymbol{b}\in \mathbb{R}^{3\times1}$, if $\mathbf{F}_{app}' = f\circ\mathbf{F}_{app}$, $\sum_{i=1}^{T}w_{i} = 1$ and  $vrr(\mathbf{r}_1,\mathbf{r}_2; \mathbf{F})<\epsilon$, we have  $vrr(\mathbf{r}_1,\mathbf{r}_2; \mathbf{F}')<K\epsilon$, where  $K =\left\| \boldsymbol{A} \right\|_2$ is the Lipschitz constant of $f$. 
	\end{proposition}
	
	\begin{proof} 
	\begin{small}
	\begin{equation}\label{eq: vrr}
		\begin{aligned}
			vrr(\mathbf{r}_1,\mathbf{r}_2; \mathbf{F}') &= \left \| {C}(\mathbf{r}_1;\mathbf{F}') -{C}(\mathbf{r}_2;\mathbf{F}')  \right \| \\ 
			&=\left \| \sum\limits_{i=1}^{T} w^{\mathbf{r}_1}_{i}f(\boldsymbol{c}^{\mathbf{r}_1}_i) -\sum\limits_{i=1}^{T} w^{\mathbf{r}_2}_{i}f(\boldsymbol{c}^{\mathbf{r}_2}_i)  \right \|\\
			&= \left \| \sum\limits_{i=1}^{T} w^{\mathbf{r}_1}_{i}(\boldsymbol{A}\boldsymbol{c}^{\mathbf{r}_1}_i+\boldsymbol{b}) -\sum\limits_{i=1}^{T} w^{\mathbf{r}_2}_{i}(\boldsymbol{A}\boldsymbol{c}^{\mathbf{r}_2}_i+\boldsymbol{b})  \right \| \\
			&= \left \| \sum\limits_{i=1}^{T} w^{\mathbf{r}_1}_{i}\boldsymbol{A}\boldsymbol{c}^{\mathbf{r}_1}_i -\sum\limits_{i=1}^{T} w^{\mathbf{r}_2}_{i}\boldsymbol{A}\boldsymbol{c}^{\mathbf{r}_2}_i  \right \| \\
			&= \left \| \boldsymbol{A} \left( \sum\limits_{i=1}^{T} w^{\mathbf{r}_1}_{i}\boldsymbol{c}^{\mathbf{r}_1}_i -\sum\limits_{i=1}^{T} w^{\mathbf{r}_2}_{i}\boldsymbol{c}^{\mathbf{r}_2}_i \right) \right \| \\
			&\leq \left\| \boldsymbol{A} \right\| \left \|  \sum\limits_{i=1}^{T} w^{\mathbf{r}_1}_{i}\boldsymbol{c}^{\mathbf{r}_1}_i -\sum\limits_{i=1}^{T} w^{\mathbf{r}_2}_{i}\boldsymbol{c}^{\mathbf{r}_2}_i \right \|
			\\
			&=\left\| \boldsymbol{A} \right\| vrr(\mathbf{r}_1,\mathbf{r}_2; \mathbf{F}) \\
			&< K \epsilon \notag
		\end{aligned}
	\end{equation}	
	\end{small}
	\end{proof}

	\begin{lemma}\label{lemma 1}
	Given $f=f_l \circ\cdots \circ f_1$,  $f_j(x) = \boldsymbol{A}_jx+\boldsymbol{b}$ if $j = l$ and  $\sigma(\boldsymbol{A}_{j}x)$ otherwise, where $\sigma$ is a $1$-Lipschitz function. Then $K = \Pi^l_{j=1}\left\| \boldsymbol{A}_{j} \right\|_{2} $ is the Lipschitz constant of $f$.
	\end{lemma}
	\begin{proof}
		Suppose that inputs $x$, $y$ belong to the domain of $f_{j}$, 
		\begin{equation}
			\begin{aligned}
			\left\| f_j(x) - f_{j}(y) \right\| &\leq  \left\| \sigma(\boldsymbol{A}_{j}x) - \sigma(\boldsymbol{A}_{j}y) \right\|  \\	
			 &\leq \left\|  \boldsymbol{A}_jx  - \boldsymbol{A}_jy \right\| \\
				&\leq \left\| \boldsymbol{A}_j \right\| \left\|  x  -  y \right\|.
			\end{aligned}
		\end{equation}
		When $l = 2$, the claim is clearly valid. The remaining cases can be easily proved by induction.
	\end{proof}

	\begin{proposition}\label{proposition 2}
		Considering $f=f_l \circ\cdots \circ f_1$,  $f_j(x) = \boldsymbol{A}_jx+\boldsymbol{b}$ if $j = l$ and  $\sigma(\boldsymbol{A}_{j}x)$ otherwise, where $\sigma=\max(0,x)$. If $\mathbf{F}_{app}' = f\circ\mathbf{F}_{app}$, $\sum_{i=1}^{T}w_{i} = 1$ and $\max_{i=1,\dots,T}\left\|  w^{\mathbf{r}_1}_{i}\boldsymbol{c}^{\mathbf{r}_1}_i - w^{\mathbf{r}_2}_{i}\boldsymbol{c}^{\mathbf{r}_2}_i\right\| < \epsilon/T$,  we have $vrr(\mathbf{r}_1,\mathbf{r}_2; \mathbf{F}')<K\epsilon$, where $K = \Pi^l_{j=1}\left\| \boldsymbol{A}_{j} \right\|_{2} $ is the Lipschitz constant of $f$. 
	\end{proposition}
	\begin{proof}
    \begin{small}
    Note that $\forall a\in \mathbb{R}^{+}$ and $1\leq j < l$, $af_j(x)=a\sigma(\boldsymbol{A}_{j}x)=\sigma(a\boldsymbol{A}_{j}x)=f_j(ax)$. Denoting $f^{j}=f_j \circ\cdots \circ f_1$, we can get the following derivation: 
    \begin{equation}
        \begin{aligned}
            af^{j}(x)&=a\sigma(\boldsymbol{A}_{j}f^{j-1}(x)) = \sigma(a\boldsymbol{A}_{j}f^{j-1}(x)) \\
            &=  \sigma(a\boldsymbol{A}_{j}\sigma(\boldsymbol{A}_{j-1}f^{j-2}(x))) \\
            &=  \sigma(\boldsymbol{A}_{j}\sigma(a\boldsymbol{A}_{j-1}f^{j-2}(x))) \\
            &\cdots \\
            & = f^{j}(ax).
        \end{aligned}
    \end{equation}
    Because the weights are always non-negative in the volume rendering integral,
    we further have
		\begin{equation}\label{eq: vrr}
		\begin{aligned}
			vrr(\mathbf{r}_1,\mathbf{r}_2; \mathbf{F}') &= \left \| {C}(\mathbf{r}_1;\mathbf{F}') -{C}(\mathbf{r}_2;\mathbf{F}')  \right \| \\ 
			&=\left \| \sum\limits_{i=1}^{T} w^{\mathbf{r}_1}_{i}f(\boldsymbol{c}^{\mathbf{r}_1}_i) -\sum\limits_{i=1}^{T} w^{\mathbf{r}_2}_{i}f(\boldsymbol{c}^{\mathbf{r}_2}_i)  \right \|\\
			&=\left \| \sum\limits_{i=1}^{T} w^{\mathbf{r}_1}_{i}\boldsymbol{A}_{l}f^{l-1}(\boldsymbol{c}^{\mathbf{r}_1}_i) -\sum\limits_{i=1}^{T} w^{\mathbf{r}_2}_{i}\boldsymbol{A}_{l}f^{l-1}(\boldsymbol{c}^{\mathbf{r}_2}_i)  \right \|\\
    		&=\left \| \sum\limits_{i=1}^{T} \boldsymbol{A}_{l}f^{l-1}(w^{\mathbf{r}_1}_{i}\boldsymbol{c}^{\mathbf{r}_1}_i) -\sum\limits_{i=1}^{T} \boldsymbol{A}_{l}f^{l-1}(w^{\mathbf{r}_2}_{i}\boldsymbol{c}^{\mathbf{r}_2}_i)  \right \|\\
			&\leq \sum\limits_{i=1}^{T}\left \|  \boldsymbol{A}_{l}f^{l-1}(w^{\mathbf{r}_1}_{i}\boldsymbol{c}^{\mathbf{r}_1}_i) - \boldsymbol{A}_{l}f^{l-1}(w^{\mathbf{r}_2}_{i}\boldsymbol{c}^{\mathbf{r}_2}_i)  \right \|.
			\\
		\end{aligned}
	\end{equation}	
    Based on above inequality and Lemma \ref{lemma 1}, we have
\begin{equation}\label{eq: vrr}
		\begin{aligned}
			&\ \ \ \ \left\|  \boldsymbol{A}_{l}f^{l-1}(w^{\mathbf{r}_1}_{i}\boldsymbol{c}^{\mathbf{r}_1}_i) - \boldsymbol{A}_{l}f^{l-1}(w^{\mathbf{r}_2}_{i}\boldsymbol{c}^{\mathbf{r}_2}_i) \right\|  \\
		&\leq \left\| \boldsymbol{A}_{l} \right\| \left\|  f^{l-1}(w^{\mathbf{r}_1}_{i}\boldsymbol{c}^{\mathbf{r}_1}_i) - f^{l-1}(w^{\mathbf{r}_2}_{i}\boldsymbol{c}^{\mathbf{r}_2}_i) \right\| \\
			&\leq \prod^{l}_{j=1} \left\| \boldsymbol{A}_{i} \right\| \ \left\|  w^{\mathbf{r}_1}_{i}\boldsymbol{c}^{\mathbf{r}_1}_i - w^{\mathbf{r}_2}_{i}\boldsymbol{c}^{\mathbf{r}_2}_i \right\|\\
			&= K \ \left\|  w^{\mathbf{r}_1}_{i}\boldsymbol{c}^{\mathbf{r}_1}_i - w^{\mathbf{r}_2}_{i}\boldsymbol{c}^{\mathbf{r}_2}_i \right\|.
		\end{aligned} 
	\end{equation}
	Therefore, 
	\begin{equation}\label{eq: vrr}
		\begin{aligned}
			vrr(\mathbf{r}_1,\mathbf{r}_2; \mathbf{F}') &\leq \sum\limits_{i=1}^{T} K \left\|  w^{\mathbf{r}_1}_{i}\boldsymbol{c}^{\mathbf{r}_1}_i - w^{\mathbf{r}_2}_{i}\boldsymbol{c}^{\mathbf{r}_2}_i \right\| \\
			&< K \sum\limits_{i=1}^{T} \epsilon/T = K\epsilon \\
		\end{aligned} 
	\end{equation}	
    \end{small}
	\end{proof}

		\begin{lemma} \label{lemma 2}
		$\mathbf{F}_{app}({\boldsymbol{x}, \boldsymbol{d}})$ = $\mathbf{F}_{sh}(\boldsymbol{x})\Gamma(\boldsymbol{d})+\boldsymbol{v}$, where $\Gamma(\boldsymbol{d}):\mathbb{R}^{2\times1} \rightarrow \mathbb{R}^{\ell\times1}$ is the spherical harmonic basis function, $\mathbf{F}_{sh}(\boldsymbol{x}) :\mathbb{R}^{3\times1} \rightarrow \mathbb{R}^{3\times \ell}$ is the coefficient function, and $\boldsymbol{v} \in \mathbb{R}^{3\times1}$. Given $\boldsymbol{A} \in \mathbb{R}^{3\times3}$, $\boldsymbol{b} \in \mathbb{R}^{3\times1}$, then
		$\boldsymbol{A}\mathbf{F}_{app}({\boldsymbol{x}, \boldsymbol{d}})+\boldsymbol{b} \Leftrightarrow 
		\boldsymbol{A}\mathbf{F}_{sh}({\boldsymbol{x}})+2\sqrt{\pi}[\boldsymbol{A}\boldsymbol{v}+\boldsymbol{b}-\boldsymbol{v},\boldsymbol{0}]$.
	\end{lemma}
	\begin{proof}
		\begin{equation}\label{eq: sh}
			\begin{aligned}
				\boldsymbol{A}\mathbf{F}_{app}({\boldsymbol{x}, \boldsymbol{d}})+\boldsymbol{b} &= 
				\boldsymbol{A}(\mathbf{F}_{sh}(\boldsymbol{x})\Gamma(\boldsymbol{d})+\boldsymbol{v}) +\boldsymbol{b} \\
				&= \boldsymbol{A}\mathbf{F}_{sh}(\boldsymbol{x})\Gamma(\boldsymbol{d})+\boldsymbol{A}\boldsymbol{v} +\boldsymbol{b} \\
				&= \boldsymbol{A}\mathbf{F}_{sh}(\boldsymbol{x})\Gamma(\boldsymbol{d})+\boldsymbol{A}\boldsymbol{v} +\boldsymbol{b}. \\
			\end{aligned}
		\end{equation}
		Because the first component of the spherical harmonic basis function outputs a constant value $\frac{1}{2\sqrt{\pi}}$,  we have
		\begin{equation}
			\begin{aligned}
				&(\boldsymbol{A}\mathbf{F}_{sh}({\boldsymbol{x}})+2\sqrt{\pi}[\boldsymbol{A}\boldsymbol{v}+\boldsymbol{b}-\boldsymbol{v},\boldsymbol{0}])\Gamma(\boldsymbol{d})+\boldsymbol{v}\\
				& = \boldsymbol{A}\mathbf{F}_{sh}\Gamma(\boldsymbol{d}) + \boldsymbol{A}\boldsymbol{v}+\boldsymbol{b}-\boldsymbol{v} + \boldsymbol{v} \\
				& =  \boldsymbol{A}\mathbf{F}_{sh}\Gamma(\boldsymbol{d}) + \boldsymbol{A}\boldsymbol{v}+\boldsymbol{b}
			\end{aligned}	
		\end{equation}
	\end{proof}
	\noindent\textbf{Remark of Lemma \ref{lemma 2}}. Similarly, it  can prove $\boldsymbol{A}\mathbf{F}_{sh}({\boldsymbol{x}})+[\boldsymbol{b},\boldsymbol{0}] \Leftrightarrow 
	 \boldsymbol{A}\mathbf{F}_{app}({\boldsymbol{x}, \boldsymbol{d}})+\frac{\boldsymbol{b}}{2\sqrt{\pi}} + \boldsymbol{v}-\boldsymbol{A}\boldsymbol{v}$.
	 
	 \begin{proposition}\label{proposition 3}
	 	Considering  $f(x) = \boldsymbol{A}x+\boldsymbol{b}$, $\boldsymbol{A}\in \mathbb{R}^{3\times \ell}$, $\boldsymbol{b}\in \mathbb{R}^{3\times \ell}$, if $\mathbf{F}_{sh}' = f\circ\mathbf{F}_{sh}$, $\sum_{i=1}^{T}w_{i} = 1$, $vrr(\mathbf{r}_1,\mathbf{r}_2; \mathbf{F})<\epsilon_1$ and $\left\| \Gamma(\boldsymbol{d^{\textbf{r}_{1}}}) -\Gamma(\boldsymbol{d^{\textbf{r}_{2}}}) \right\|<\epsilon_2$, we have  $vrr(\mathbf{r}_1,\mathbf{r}_2; \mathbf{F}')<K_1\epsilon_1$ + $K_2\epsilon_2$, where  $K_1 =\left\| \boldsymbol{A} \right\|_2$ and $K_2 =\left\| \boldsymbol{b} \right\|_2$. Moreover, if $\boldsymbol{b}$ vanishes except for the first column (\textit{i.e.}, the form in above remark), $vrr(\mathbf{r}_1,\mathbf{r}_2; \mathbf{F}')<K_1\epsilon_1$.
	 \end{proposition}
	
		\begin{proof} 
		\begin{small}
			\begin{equation}\label{eq: vrr}
				\begin{aligned}
					&vrr(\mathbf{r}_1,\mathbf{r}_2; \mathbf{F}')\\
					 &= \left \| {C}(\mathbf{r}_1;\mathbf{F}') -{C}(\mathbf{r}_2;\mathbf{F}')  \right \| \\ 
					&= \left \| \sum\limits_{i=1}^{T} w^{\mathbf{r}_1}_{i}\textbf{F}'(\boldsymbol{x}^{\textbf{r}_{1}}_{i})\Gamma(\boldsymbol{d^{\textbf{r}_{1}}}) -\sum\limits_{i=1}^{T} w^{\mathbf{r}_2}_{i}\textbf{F}'(\boldsymbol{x}^{\textbf{r}_{2}}_{i})\Gamma(\boldsymbol{d^{\textbf{r}_{2}}})  \right \| \\
					&\leq \left \| \sum\limits_{i=1}^{T} w^{\mathbf{r}_1}_{i}\boldsymbol{A}\textbf{F}(\boldsymbol{x}^{\textbf{r}_{1}}_{i})\Gamma(\boldsymbol{d^{\textbf{r}_{1}}}) -\sum\limits_{i=1}^{T} w^{\mathbf{r}_2}_{i}\boldsymbol{A}\textbf{F}(\boldsymbol{x}^{\textbf{r}_{2}}_{i})\Gamma(\boldsymbol{d^{\textbf{r}_{2}}}) \right \|  \\
					&\ \ \ \ +  \left \| \boldsymbol{b}\Gamma(\boldsymbol{d^{\textbf{r}_{1}}}) -\boldsymbol{b}\Gamma(\boldsymbol{d^{\textbf{r}_{2}}}) \right \| \\
					&\leq  \left\| \boldsymbol{A} \right\| vrr(\mathbf{r}_1,\mathbf{r}_2; \mathbf{F})  +  \left \| \boldsymbol{b}\right \| \left\| \Gamma(\boldsymbol{d^{\textbf{r}_{1}}}) -\Gamma(\boldsymbol{d^{\textbf{r}_{2}}})  \right\|\\
					&<  K_1\epsilon_1 + K_2\epsilon_2.\\
				\end{aligned}
			\end{equation}	
		If $\boldsymbol{b}$ vanishes except for the first column, $ \left \| \boldsymbol{b}\Gamma(\boldsymbol{d^{\textbf{r}_{1}}}) -\boldsymbol{b}\Gamma(\boldsymbol{d^{\textbf{r}_{2}}}) \right \| =0$, thus   $vrr(\mathbf{r}_1,\mathbf{r}_2; \mathbf{F}')<K_1\epsilon_1$.
	\end{small}
	\end{proof}

\noindent\textbf{Remarks}. Prop.~\ref{proposition 3} extends Lipschitz-constrained linear mapping in Prop.~\ref{proposition 1} from appearance representation to spherical harmonics. To prove the bound of Lipschitz MLP applied to spherical harmonics, some fussy assumptions are further required, and the proof will be trivial to repeat the above proving processes. We believe the three propositions have exhibited the intuition and importance of Lipschitz transformations for this task.
\section{More results}
For comprehensive analysis and evaluation, we have supplied a video\footnote{\url{https://www.youtube.com/watch?v=1ft8Mev3RmE}} in the supplementary materials,  which contains the continuous novel views of multiple scenes stylized with various references.  It can be observed that, both $\text{WCT}^2$ and CCPL create noises and disharmony to affect the photorealism of video. In specific, $\text{WCT}^2$ is likely to sharpen the edges excessively that produces artificial boundaries around edges (\textit{e.g.}, the trex and room scenes). It also generates noticeable noises in some stylized scenes. The results of CCPL usually have richer colors that enhances the visual effects. However, the variegated colors acceptable in a still image may be harmful to 3D scenes. For example, in the trex and fortress scenes, the interframe variations of colors results in artifacts and unconsistency of videos. In the flower scene, due to the unconsistency, the colorful leaves and flowers seem to be unrealistic and flickering. In contrast, LipRF can alleviate these downsides to generate more consistent and photorealistic stylized novel views while transferring the color style. The videos of LipRF are more like camera shots to meet the requirement of photorealistic 3D scene stylization.
